\documentclass{article}

\usepackage{arxiv}

\usepackage[utf8]{inputenc} % allow utf-8 input
\usepackage[T1]{fontenc}    % use 8-bit T1 fonts
\usepackage{hyperref}       % hyperlinks
\usepackage{url}            % simple URL typesetting
\usepackage{booktabs}       % professional-quality tables
\usepackage{amsfonts}       % blackboard math symbols
\usepackage{nicefrac}       % compact symbols for 1/2, etc.
\usepackage{microtype}      % microtypography
\usepackage{lipsum}		% Can be removed after putting your text content
\usepackage{graphicx}
\usepackage{doi}
\usepackage{amsmath}
\usepackage{amssymb}
\usepackage{algorithmic}
\usepackage{textcomp}
\usepackage{enumitem}
\usepackage{xcolor}
\usepackage{subcaption}
\usepackage{mwe}  % For placeholder images
\usepackage{adjustbox}
\usepackage{lipsum}
\usepackage{wrapfig}
\newcommand{\parafango}[1]{\textbf{#1.}$\ \ $}

\newif
\ifdraft
\draftfalse

\ifdraft

\definecolor{refkey}{rgb}{0,.6,0}
\definecolor{labelkey}{cmyk}{0, 1, 0, 0}
\fi

\newtheorem{remark}{Remark}

\newtheorem{theorem}{Theorem}
\newtheorem{proposition}{Proposition}
\newtheorem{corollary}{Corollary}

\def\R{\mathbb{R}}
\def\C{\mathbb{C}}

\def\Id{\mathrm{Id}}
\def\diag{\mathrm{diag}}
\def\spec{\mathop{\rm spec}}
\def\Re{\mathop{\rm Re}\nolimits}

\def\sigmavec{{\boldsymbol{\sigma}}}

\newenvironment{proof}[1][Proof]{\noindent\textit{#1:} }{\hfill$\square$\par} % WARNING THIS 

\title{Generative System Dynamics \\ in Recurrent Neural Networks \footnotemark[1]\hspace{0.5em}\footnotemark[2]}

\author{
\href{https://scholar.google.com/citations?user=yY2OyccAAAAJ}{Michele Casoni\textsuperscript{1}},
\hspace{1mm}
\href{https://scholar.google.com/citations?user=7zkmhw0AAAAJ}{Tommaso Guidi\textsuperscript{1,2}},
\hspace{1mm}
\href{https://scholar.google.com/citations?user=SBgRcNUAAAAJ}{Alessandro Betti\textsuperscript{3}},
\hspace{1mm}
\href{https://scholar.google.com/citations?user=_HHu1MQAAAAJ}{Stefano Melacci\textsuperscript{1}},
\hspace{1mm}
\href{https://scholar.google.com/citations?user=wBMRRk0AAAAJ}{Marco Gori\textsuperscript{1}}
\\
(1) Department of Information Engineering and Mathematics, University of Siena, Siena, Italy\\
(2) Department of Information Engineering, University of Florence, Florence, Italy\\
(3) IMT School for Advanced Studies Lucca, Lucca, Italy\\
\texttt{m.casoni@student.unisi.it}, \texttt{tommaso.guidi@unifi.it}, \texttt{alessandro.betti@imtlucca.it} \\ \texttt{stefano.melacci@unisi.it}, \texttt{marco.gori@unisi.it}\\
}

\date{}

\renewcommand{\headeright}{}
\renewcommand{\undertitle}{}
\renewcommand{\shorttitle}{Generative System Dynamics in Recurrent Neural Networks}

\begin{document}
\maketitle
\footnotetext[1]{Copyright © 2025 IEEE.  Personal use of this material is permitted.  Permission from IEEE must be obtained for all other uses, in any current or future media, including reprinting/republishing this material for advertising or promotional purposes, creating new collective works, for resale or redistribution to servers or lists, or reuse of any copyrighted component of this work in other works.}
\footnotetext[2]{Accepted for publication at the 2025 IEEE International Joint Conference on Neural Networks (IJCNN).}
\begin{abstract}
In this study, we investigate the continuous time dynamics of Recurrent Neural Networks (RNNs), focusing on systems with nonlinear activation functions. The objective of this work is to identify conditions under which RNNs exhibit perpetual oscillatory behavior, without converging to static fixed points. We establish that skew-symmetric weight matrices are fundamental to enable stable limit cycles in both linear and nonlinear configurations. We further demonstrate that hyperbolic tangent-like activation functions (odd, bounded, and continuous) preserve these oscillatory dynamics by ensuring motion invariants in state space. Numerical simulations showcase how nonlinear activation functions not only maintain limit cycles, but also enhance the numerical stability of the system integration process, mitigating those instabilities that are commonly associated with the forward Euler method. The experimental results of this analysis highlight practical considerations for designing neural architectures capable of capturing complex temporal dependencies, i.e., strategies for enhancing memorization skills in recurrent models.
\end{abstract}

\section{Introduction}\label{sec:intro}
Recurrent neural networks (RNNs) are fundamental tools for modeling temporal dependencies in sequential data, finding applications in fields such as time-series forecasting, control systems, and natural language processing, recently reconsidered also when dealing with very long sequences \cite{tiezzi2024resurgencerecurrentmodelslong}. However, training RNNs in an online fashion while learning long-term dependencies presents significant challenges. The problem of vanishing and exploding gradients \cite{bengio1994learning,bengio1993credit} remains a key obstacle, causing models to lose track of the past information or become unstable when gradients are propagated over time. Recent advancements in sequence modeling have introduced Deep State-Space Models (SSMs) \cite{gu2021efficiently,smith2022simplified}, which leverage state-space representations of dynamic systems to capture long-range dependencies and exhibit strong generalization capabilities. Initializing these models with complex eigenvalues near the unit circle has been shown to enhance their ability to store and propagate information over long horizons \cite{orvieto2023universality}. Despite these advances, learning long-term dependencies online---without storing or revisiting historical data---remains a significant challenge \cite{casoni2024pitfalls}.
Moreover, when considering online lifelong learning scenarios \cite{parisi2020online}, the problem of long-range dependencies is particularly relevant not only during training (when computing gradients) but also during inference \cite{gunasekara2023survey}. Unlike models designed to make predictions on separate sequences, agents operating on a continuous stream of information must sustain non-trivial behavior even in the absence of external stimuli.
This highlights the significance of \emph{perpetual generation}, which involves autonomously producing coherent, infinite-length sequences by learning the dynamics of a single data stream. Crucially, this capability allows models to evolve their state without collapsing into trivial solutions such as static fixed points or vanishing outputs. The importance of perpetual generation lies in its ability to sustain non-trivial dynamics, enabling systems to effectively capture long-term dependencies while avoiding degradation over time. A fundamental aspect of achieving perpetual generation is understanding and controlling the “free dynamics” of recurrent neural networks (RNNs), i.e., their behavior when left entirely uninfluenced by external inputs or supervision. Studying these dynamics reveals architectural conditions necessary for stability and coherence, preventing constant or vanishing outputs that would otherwise limit the network’s ability to retain and process information across extended timescales.

\parafango{Contributions} In this work, we focus on identifying design principles for RNNs that confine the free evolution of their state to stable limit cycles, thereby guaranteeing perpetual oscillations without convergence to fixed points or divergence to unstable behaviors. Specifically, we explore architectural choices involving activation functions and spectral properties of the state matrix that enable such oscillatory dynamics.
%%%
The contributions of this paper are summarized as follows: ($i$) we show that skew-symmetric weight matrices guarantee perpetual oscillatory behavior in both linear and nonlinear RNNs; ($ii$)  we demonstrate that even with nonlinear activation functions
this behaviors can be preserved giving a complete theoretical analysis for the case of the hyperbolic tangent; ($iii$) we validate our theoretical findings through numerical simulations, providing visual insights into the state trajectories and oscillatory patterns.

\section{State Space Systems Dynamics}\label{sec:ss-dyn}
Dynamic systems of ordinary differential equations (ODEs) are often used to describe the evolution of processes over time. Depending on the structure of the ODEs and the parameters involved, these systems can exhibit different behaviors, from periodic oscillations to complex chaotic patterns.

%\subsection{System Dynamics Overview}
\parafango{System Dynamics Overview}
Classic dynamic systems often present \emph{oscillations}, where the system naturally settles into periodic orbits known as \emph{limit cycles}, regardless of initial conditions. One example is the Van der Pol oscillator \cite{vanderPol1926oscillations}, a nonlinear system characterized by a damping parameter that produces stable limit cycles without chaotic behavior. Similarly, the FitzHugh-Nagumo model \cite{FITZHUGH1961445,4066548}, commonly used to describe neural dynamics, supports stable oscillations under specific parameter settings. Another important system is the Lotka-Volterra predator-prey model \cite{lotka1925elements,VolterraFluctuationsIT}, where populations of prey and predators oscillate periodically under idealized conditions. Other systems exhibit \emph{chaotic} dynamics characterized by sensitivity to initial conditions and parameters selection. The Lorenz system \cite{Lorenz1963DeterministicNF}, for instance, demonstrates deterministic chaos when specific parameter values are applied. The Rössler system \cite{ROSSLER1976397} also displays chaotic attractors within suitable parameter ranges. Additionally, Chua’s circuit \cite{chua1983genesis} is well-known for generating chaotic behavior in electronic circuits due to its piecewise-linear functions. Some systems are capable of transitioning between periodic oscillations and chaos, depending on parameters variations. The Duffing oscillator \cite{duffing1918erzwungene} illustrates this phenomenon by exhibiting both limit cycles and chaotic dynamics as parameters such as damping and external forcing amplitude are changed. Similarly, the Hindmarsh-Rose model \cite{hindmarsh1984neuronal} for neural dynamics can transition between periodic spiking and chaotic bursting under varying external stimuli. Modeling such dynamics can be challenging due to the complexity of many real-world systems and the need to account for nonlinearities and feedback mechanisms. A promising approach to modeling these dynamics is through the use of Recurrent Neural Networks. RNNs are well-suited to capture the temporal dependencies and feedback present in many dynamic systems, including those exhibiting oscillations and chaotic behaviors.

%\subsection{Modeling dynamics with RNNs}
\parafango{Modeling dynamics with RNNs}
RNNs are often represented by continuous-time dynamic systems of the following form\footnote{For the sake of definiteness and simplicity, and since this is not a crucial point in our work, we will consider classical solutions of ODE systems---i.e., sufficiently smooth functions for which differential conditions can be expressed pointwise for all  $t$.}
\begin{equation}
\label{eq:ctrnn-general}
\begin{cases}
    \dot x (t) = \sigmavec \left(A x (t) + B u(t) \right),&t\geq0 \\
    y(t) = \sigmavec_{\text{out}} \left( C x(t) + D u(t)\right),&t\geq0 \\
    x(0) = \bar x,
\end{cases}
\end{equation}
where $t\mapsto x(t)\in\R^n$ is the \emph{state} trajectory of the RNN, $\bar x$
is the initial condition of the state, $t\mapsto u(t)\in\R^d$ is the input provided to the network and $t\mapsto y(t)\in\R^m$ is its output. Matrices $A \in \mathbb{R}^{n \times n}$, $B \in \mathbb{R}^{n \times d}$, $C \in \mathbb{R}^{m \times n}$ and $D \in \mathbb{R}^{m \times d}$ are the parameters of the model, with $A$ and $C$ usually called the \emph{state} and \emph{output} matrix, respectively. The functions $\sigmavec\colon \mathbb{R}^n \to \mathbb{R}^n$ and $\sigmavec_\text{out}\colon \mathbb{R}^m \to \mathbb{R}^m$ apply the \emph{activation functions} $\sigma\colon \mathbb{R} \to \mathbb{R}$ and $\sigma_{\text{out}}\colon \mathbb{R} \to \mathbb{R}$ component-wise to a vector $v \in \mathbb{R}^n$ (or $\mathbb{R}^m$), respectively. Specifically, the component-wise activation function acts as follows:
\begin{equation}
    (\sigmavec(v))_i = \sigma(v_i),\quad \forall v\in\R^n 
\end{equation}
for every component $i = 1,\dots, n$. The same with $m$ instead of $n$ holds for  
$\sigmavec_{\rm out}$. A feedback mechanism can be also added in the state dynamics, leading to
\begin{equation}
\label{eq:ctrnn-feedback}
    \dot x (t) =  - \alpha x(t) + \sigmavec \left(A x (t) + B u(t) \right),\quad \alpha>0.
\end{equation}
In all the equations, we avoided to insert bias terms in the activations of the neurons. In this work, we focus on the free evolution of RNNs over time, assuming no input is provided and the model parameters are time-independent. Therefore, we consider a reduced form of the dynamic system in Eq.~\eqref{eq:ctrnn-general}, where the matrices $B$ and $D$ are set to zero, and $A$ and $C$ are constant. Additionally, since our primary interest lies in examining the evolution of the state $x(t)$ over time, we omit the output mapping $y(t)$ in the following analysis. These simplifications lead us to consider the system:
\begin{equation}
\label{eq:ctrnn-no-input}
\begin{cases}
    \dot x (t) = \sigmavec \left(A x (t) \right), \\
    x(0) = \bar x.
\end{cases}
\end{equation}
The continuous-time dynamical system in Eq.~\eqref{eq:ctrnn-no-input} can be discretized using a variety of schemes, the simplest of which is the 
\emph{forward (or esplicit) Euler method} with a step size $\tau$, yielding  
\begin{equation}
\label{eq:ctrnn-euler}
    x_{k+1} = x_k + \tau \sigmavec(A x_k),\quad x_0=\bar x.
\end{equation}  
%%%%
This discretization will be used to evolve the state equations in the experiments (Section~\ref{sec:exp}). In what follows, we will present conditions that ensure the RNN dynamics exhibit perpetual oscillations without converging to static fixed points, as already said in Section~\ref{sec:intro}. First, we will review some well-known results for Linear RNNs (Section~\ref{sec:linear-ctrnns}) and then transition to the nonlinear case (Section~\ref{sec:non-linear-ctrnns}).

\section{Linear RNNs}\label{sec:linear-ctrnns}
In Linear RNNs, the activation function $\sigma$ is the identity map, i.e., $\sigma = \operatorname{id}$. Therefore, the dynamic system in Eq.~\eqref{eq:ctrnn-no-input} simplifies to
\begin{equation}
\label{eq:ctrnn-linear}
\begin{cases}
    \dot x (t) = A x (t), \\
    x(0) = \bar x.
\end{cases}
\end{equation}
The exact solution of the state equation with initial condition $x(0)=\bar x$ is $x(t)=e^{tA}\bar x$ \cite{coddington1955theory}, where $e^{tA}$ is the matrix exponential defined as
\begin{equation}
\label{eq:matrix-exp}
    e^{tA} := \sum_{k=0}^{+ \infty} \frac{(tA)^{k}}{k!} =\Id + tA + \frac{(tA)^2}{2!} + \frac{(tA)^3}{3!}
    + \cdots.
\end{equation}
The stability of the solution $x(t) = e^{tA} \bar{x}$ is
determined by the spectrum (set of the eigenvalues) of $A$ \cite{thompson2013ordinary}, 
denoted as $\spec(A)$\footnote{It consists of $n$
eigenvalues counted with their multiplicity.}. In particular, if $\lambda \in \spec(A)$ is an eigenvalue of $A$, the solution of the linear dynamic system in Eq.~\eqref{eq:ctrnn-linear} exhibits the following behaviors \cite[p.~308]{thompson2013ordinary}:

\begin{theorem}
\label{thm:stability-cont}
Let $A\in\R^{n \times n}$  a matrix over the complex field
and consider the linear system in Eq.~\eqref{eq:ctrnn-linear}.
The stability of the solution $x(t) = e^{tA} \bar{x}$ is
determined by $\spec(A)$. In particular 
the following facts hold true:
\begin{enumerate}
\item  The system is \emph{asymptotically stable} if\/
$\Re(\lambda) < 0$ for all $\lambda\in\spec(A)$.
\item The system is \emph{marginally stable} if
$\Re(\lambda) \leq 0$ for all $\lambda\in\spec(A)$
and, for every $\lambda$  such that $\Re(\lambda) = 0 \), $m_a(\lambda)= m_g(\lambda)$ (algebraic multiplicity 
$m_a(\lambda)$ is equal to the geometric multiplicity $m_g(\lambda)$).\footnote{This also means that 
the Jordan Blocks associated to $\lambda$ are all $1\times 1$.}
\item The system is \emph{unstable} in all other cases, that is if there  exists $\lambda\in \spec(A)$ such that $\Re(\lambda) > 0$, 
or if\/ $\Re(\lambda_n) = 0$ and $m_a(\lambda)\ne m_g(\lambda)$.
\end{enumerate}
\end{theorem}
Since we are interested in conditions that guarantee perpetual oscillations of the state that can translate to generative skills, we focus on the case of marginal stability. This is explained as follows: for any (real or complex) matrix $A$, there exists an invertible matrix $P$ such that $A = P J P^{-1}$, where $J$ is the Jordan form of $A$. Then, from the definition of matrix exponential and being $P$ invertible, it immediately follows that
\begin{equation}
    e^{tA} = P e^{tJ} P^{-1},
\end{equation}
where $e^{tJ}$ can be computed block by block. For a Jordan block \cite[p.~164]{horn2012matrix} $J_n$ associated with the eigenvalue $\lambda_n$, we have
\begin{equation}
    e^{t J_n} = e^{\lambda_n t}
\begin{bmatrix}
1 & t & \frac{t^2}{2!} & \cdots \\
0 & 1 & t & \cdots \\
0 & 0 & 1 & \cdots \\
\vdots & \vdots & \vdots & \ddots
\end{bmatrix}.
\end{equation}
Focusing on marginal stability, if $\text{Re}(\lambda_n) < 0 $, $e^{\lambda_n t} \to 0$. In this case, $e^{tJ_n}$ is bounded, and $\|x(t)\|$ remains bounded. If $\text{Re}(\lambda_n) = 0$, $\lambda_n$ is purely imaginary, i.e., $\lambda_n = i \omega_n$ for some real $\omega_n$. Moreover, since $m_a(\lambda)= m_g(\lambda)$, $J_n =\lambda_n$ (the Jordan block is a $1\times1$ block). Then, by the Euler formula,
\begin{equation}
    e^{\lambda_nt} = e^{i \omega_n t} = \cos(\omega_nt) + i\sin(\omega_n t).
\end{equation}
This implies that the solution oscillates perpetually. To guarantee marginal stability, a possible condition is for the matrix $A$ to be \emph{skew-symmetric}, i.e., satisfying $A = -A'$, where the prime denote the transpose operation. Indeed, the eigenvalues of a skew-symmetric matrix are purely imaginary (see Corollary 2.5.11 in \cite{horn2012matrix}), ensuring marginal stability for the solution of Eq.~\eqref{eq:ctrnn-linear}. Another important property of skew-symmetric matrices that will be used in the next Section is stated in Proposition~\ref{prop:xAx-skew}. Let $\overline{z}$ denote the complex conjugate of a complex number $z \in \mathbb{C}$, and let $M^*$ denote the conjugate transpose of a matrix $M$. Additionally, let $\langle \cdot, \cdot \rangle$ denote the standard inner product in the complex vector space $\mathbb{C}^{n}$.
\begin{proposition}
\label{prop:xAx-skew}
 If \(A \in \mathbb{R}^{n \times n}\) is skew-symmetric, then
\[
\langle x, A x \rangle = - \overline{\langle x, A x \rangle},
\]
for every $x\in \C^{n}$.   
\end{proposition}

\begin{proof}
Since $A$ is skew-symmetric then $A =(A -A')/2$. 
Then, chosen any $x\in\C^{n}$, we have
\[\begin{aligned}
\langle x, Ax \rangle &= \frac{1}{2}\langle x, (A-A')x \rangle
=\frac{1}{2}(\langle x, Ax \rangle -\langle x, A^*x \rangle)\\
&=\frac{1}{2}(\langle A^*x, x \rangle -\langle Ax, x \rangle)
=-\langle Ax, x \rangle=-\overline{\langle x, Ax \rangle}.
\end{aligned}\]
\end{proof}

\begin{remark}\label{rem:energy-cons}
 For every $x \in \R^{n}$, $\left( x, A x \right) = 0$, where $\left( x, A x \right)$ is the standard inner product in the real vector space $\mathbb{R}^n$. This orthogonality between $x$ and $Ax$ has profound implications for the dynamics of systems governed by skew-symmetric matrices. In geometric terms, it means that the action of the linear transformation defined by $A$ rotates vectors around the origin without changing their magnitude. This behavior can be interpreted as inducing \emph{rotational dynamics} in the state space. Indeed, let us consider the inner product $\left(x(t), x(t) \right) = \|x(t)\|^2$. Its time derivative is
 \begin{equation}\label{eq:energy-cons}
     \frac{d}{dt} \|x(t)\|^2 = 2 \left(x(t), \dot x (t)\right) = 2 \left(x(t), Ax(t) \right) = 0.
 \end{equation}
 Therefore, the magnitude of $x(t)$ is constant as the system evolves. This invariance of the norm implies that the system trajectories are confined to surfaces of constant radius, which geometrically correspond to circular orbits in the state space. The preservation of the norm and the orthogonality between $x(t)$ and $Ax(t)$ are characteristic features of rotational dynamics. 
\end{remark}

\section{Nonlinear Activation function}\label{sec:non-linear-ctrnns}
In Section~\ref{sec:linear-ctrnns} we highlighted the critical role of skew-symmetric matrices for ensuring perpetual oscillations in Linear RNNs. 
That analysis heavily depended on the linearity of the model 
(Theorem~\ref{thm:stability-cont}).
%%%
In this section, instead,  we show that this property can be retained also in the presence 
of nonlinear activation functions (Eq.~\eqref{eq:ctrnn-no-input}).
In particular we analyze in details the case where $\sigmavec$ 
is the hyperbolic tangent.

\subsection{Lyapunov Stability Theorem}
In our analysis, we make three key assumptions: ($i$) the state matrix $A$ is skew-symmetric; ($ii$) the nonlinear activation function $\sigma$ is odd, bounded, and continuous, meaning $\sigma(-a) = -\sigma(a)$ for every scalar $a \in \mathbb{R}$; ($iii$) no additional dissipative terms are present. Important properties arise from the oddity and boundedness of $\sigma$ required by assumption ($ii$): the first implies that the flow induced by \( \sigmavec(Ax) \) is symmetric with respect to the origin, resulting in balanced state trajectories that do not exhibit drift; the latter ensures that the trajectories are confined to a bounded region in the state space. Examples of nonlinear activation functions satisfying conditions in assumption ($ii$) are the hyperbolic tangent $\sigma = \tanh$ and its rectified version $\sigma = \text{hardtanh}$, defined as the piece-wise linear function
\begin{equation}
\label{eq:hardtanh}
    \text{hardtanh}(a) = \begin{cases}
        a, \qquad \qquad \text{if } |a| \leq 1, \\
        \text{sgn}(a), \qquad \text{otherwise}.
    \end{cases}
\end{equation}
These two examples of odd nonlinear activation functions are used in our experimental campaign. Our objective is to demonstrate that this setup guarantees oscillatory behavior in the system dynamics. In practice, we would like to show that with skew-symmetric $A$
and $\sigma=\tanh$ it is possible to have \emph{bounded} solutions that 
do not converge to constant values, but keep moving into a limit cycle. 

While it is not possible to give a direct extension of Theorem~\ref{thm:stability-cont},
our approach will be that of extending the idea explained in Remark~\ref{rem:energy-cons},
that is using a motion invariant to show that the system makes closed orbits in 
state space. This argument, that is an energy-conservation argument, is 
closely related to the classical notion of Lyapunov stability~\cite{khalil2002nonlinear},
as a Lyapunov function can be regarded as a generalized energy function.\footnote{Remember 
that Lyapunov stability means that if we start close from an equilibrium point, 
we always remain close to that point.}
First of all, let us recall this classical result:

\begin{theorem}[Lyapunov Stability Theorem]
\label{thm:lyapunov-stability}
Consider a nonlinear dynamical system $\dot{x} = f(x)$ with $f(0)=0$. Let $H: \mathbb{R}^n \rightarrow \mathbb{R}$ be a continuously differentiable function satisfying the following conditions: 
\begin{enumerate}[label=(\roman*)]
\item $H(0) = 0$
\item $H(x) > 0$ for all $x\in\R^n\setminus\{0\}$;
\item $\Bigl(\nabla H(x), f(x) \Bigr) \leq 0$ for all $x\in\R^n$.
\end{enumerate}
Then, the zero solution $x(t)=0$ is Lyapunov stable.
\end{theorem}
Notice that the assumption $f(0)=0$ holds naturally for any odd activation function $\sigma$,
hence also for the hyperbolic tangent. The condition $\Bigl(\nabla H(x), f(x) \Bigr) \leq 0$ in Theorem~\ref{thm:lyapunov-stability},
is related to the existence of \emph{invariants of the motion} (or \emph{constants of the motion}). Indeed, given a solution $t\mapsto x(t)$ of Eq.~\eqref{eq:ctrnn-no-input},
$H: \mathbb{R}^n \rightarrow \mathbb{R}$ is an invariant of the motion if
\begin{equation}
\label{eq:lyapunov-inv}
\frac{d}{dt} H(x(t)) = \Bigl(\nabla H(x(t)), \sigma(Ax(t)) \Bigr) = 0, 
\end{equation}
where $t\mapsto x(t)\in\R^n$ is a solution to Eq.~\eqref{eq:ctrnn-no-input}.

\begin{remark}[Linear RNN with skew-symmetric matrix]
Following from Remark~\ref{rem:energy-cons},
let us consider again the linear system in Eq.~\eqref{eq:ctrnn-linear} with skew-symmetric matrix $A \in \mathbb{R}^{n \times n}$. If we define
\begin{equation}
    H(x) := \frac{1}{2} \|x\|^2,\quad \forall x\in\R^n,
\end{equation}
we can promptly see from Eq.~\eqref{eq:energy-cons} (since $A$ is skew-symmetric and Proposition~\ref{prop:xAx-skew} holds) that
\begin{equation}
\begin{aligned}
    \frac{d}{dt}H(x(t)) &= \Bigl( \nabla H(x(t)), Ax(t)\Bigr)=0
\end{aligned}
\end{equation}
This also means, in view of Theorem~\ref{thm:lyapunov-stability}, that $H$ is a Lyapunov
function of the system and that $x(t)=0$ is Lyapunov stable.
\end{remark}
Lyapunov stable points that are not asymptotically Lyapunov stable are indeed marginally 
stable points.
Our strategy now is then to show that even in the nonlinear case, with $\sigma=\tanh$, it is possible to find an invariant of the motion.

\subsection{Stability of nonlinear RNNs}
Consider the nonlinear activation function $\sigma(a) = \tanh(a)$, which satisfies the conditions on $\sigma$ of assumption ($ii$) in the previous subsection. Then we have the following result.

\begin{wrapfigure}{r}{0.3\textwidth}  % "r" = a destra, "l" = a sinistra
  \centering
  \includegraphics[width=0.28\textwidth]{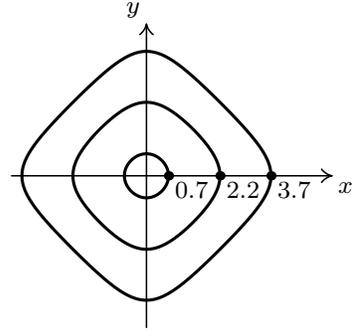}
  \caption{\label{fig:level-sets} 
  Level sets $\{(x,y)\in\mathbb{R}^2:\log(\cosh(x)\cosh(y))=\ell\}$ of Eq.~\eqref{eq:energy-2d} with $\omega=1$ for $\ell=0.2,1.5,3$.}
\end{wrapfigure}

\begin{proposition}\label{prop:2d-inv}
Let $n=2$ and suppose that $A$ is the generic skew-symmetric matrix $A=\bigl({0\atop \omega}{-\omega\atop 0}\bigr)$. Then, the function
$H\colon\R^2\to[0,+\infty)$ defined as
\begin{equation}\label{eq:energy-2d}
H(x):=\frac{1}{\omega}\log\Bigl(\cosh(\omega x_1) \cosh(\omega x_2)\Bigr),
\quad \forall x\in\R^2,
\end{equation}
is an invariant of the motion defined by Eq.~\eqref{eq:ctrnn-no-input}.
\end{proposition}
\begin{proof}
Suppose $t\mapsto x(t)\in\R^2$ is a classical solution of 
$\dot x(t)=\sigmavec(Ax(t))$. Then, 
\begin{equation}
\begin{aligned}
\frac{d}{dt} H(x(t))&=\Bigl(\nabla H(x(t)), \dot x(t)\Bigr)\\
&=\biggl(
\begin{pmatrix}
\tanh(\omega x_1(t))\\
\tanh(\omega x_2(t))
\end{pmatrix}, 
\begin{pmatrix}
-\tanh(\omega x_2(t))\\
\tanh(\omega x_1(t))
\end{pmatrix}
\biggr) \\
&=0.
\end{aligned}
\end{equation}
Hence $H$ is a constant of the motion.
\end{proof}
\medskip
This result, as in the linear case, is sufficient to prove that the orbits of the system in 
the state space are closed:

\begin{corollary}\label{cor:sublevels}
Under the same hypothesis of Proposition~\ref{prop:2d-inv}, we have that the curves 
$t\mapsto x(t)$ that solve Eq.~\eqref{eq:ctrnn-no-input} are closed smooth, without auto-intersections.
\end{corollary}
\begin{proof}
Consider a generic level set $\{H=\ell\}$. 
Since the logarithm is a smooth, strictly increasing function, the sublevels of $H$ are the same as those of 
$f(x_1,x_2)=\cosh(\omega x_1) \cosh(\omega x_2)$ (we can discard the logarithm).
The function $f$ goes to $+\infty$ as $|(x_1, x_2)| \to +\infty$ ($f$
is coercive), then 
the sublevels $\{f\le \ell\}$ are compact, in particular bounded.
The boundary of  $\{f\le \ell\}$ is  $\{f=\ell\}$ and by the implicit function theorem 
applied to $f$ (that can be applied since $\nabla f\ne0$ on the level)
we have that it is a smooth 1D manifold.
This means that the level 
$\{H=\ell\}$ is made up of smooth closed curves (loops) that do not intersect each other or themselves. Moreover, since the sublevel relative to 
$\ell$ is simply connected, it is made up of a \emph{single} closed curve.
\end{proof}

Figure~\ref{fig:level-sets} shows several level sets of  $H$ . As the cut height increases ($\ell >$ 0 grows), the orbits gradually take on a more squared shape. These curves, derived theoretically, have also been observed experimentally (see Fig.~\ref{fig:fig1}).
%%%
We want to remark that the two dimensional case that we have just analyzed does not represent 
only an exercise, but rather a base step to tackle the problem at a generic dimension $n$.
We recall that  any real skew-symmetric matrix $A \in \mathbb{R}^{n \times n}$ with even dimension\footnote{Here we consider just even dimension, but the same result holds for 
odd dimension just by adding a zero to the diagonal of Eq.~\eqref{eq:block-diagonal}} 
$n=2d$ can be brought to a block-diagonal form by an orthogonal transformation, where the $d$ blocks are $2\times2$ skew-symmetric 
matrices of the form 
$\bigl({0\atop \omega}{-\omega\atop 0}\bigr)$ (see Corollary 2.5.11 in \cite[p.~136-137]{horn2012matrix}).
Such result in the linear setting suggests that the dynamical behavior of 
Eq.~\eqref{eq:ctrnn-linear} can be fully understood just by restricting to the following
class of skew-symmetric  matrices:
\begin{equation}\label{eq:block-diagonal}
A=\diag\left(
\begin{pmatrix}
0& -\omega_1\\
\omega_1&0
\end{pmatrix},\dots,
\begin{pmatrix}
0& -\omega_d\\
\omega_d&0
\end{pmatrix}
\right).
\end{equation}
This is not strictly true in the nonlinear setting, since we cannot easily change base in the 
ODE due to the nonlinearity. At any rate, this motivates us to study the $n$ dimensional case
restricted to matrices of the form described in Eq.~\eqref{eq:block-diagonal}.
Hence, the following result holds:

\begin{proposition}
Let  $n=2d$ and assume that the matrix $A$ is skew-symmetric 
and block diagonal of the form in Eq.~\eqref{eq:block-diagonal}.
Then, the function $H\colon\R^n\to[0,+\infty)$
defined as 
\begin{equation}\label{eq:motion-invariance-nd}
H(x):=\sum_{i=1}^d 
\frac{1}{\omega_i}\log\Bigl(\cosh(\omega_i x_{2i-1}) \cosh(\omega_i x_{2i})\Bigr),
\end{equation}
for all $x\in\R^{2d}$, is a constant of the motion~\eqref{eq:ctrnn-no-input}.
\end{proposition}

\begin{proof}
For $k=1\dots n$ and for all $x\in\R^n$, the gradient of Eq.~\eqref{eq:motion-invariance-nd} is
\[
\begin{aligned}
\frac{\partial}{\partial x_k} H(x)
&=\sum_{i=1}^d \Biggl[
\frac{\sinh(\omega_i x_{2i-1}) \cosh(\omega_i x_{2i})\delta_{k,2i-1}
}{\cosh(\omega_i x_{2i-1}) \cosh(\omega_i x_{2i})}+frac{\cosh(\omega_i x_{2i-1}) \sinh(\omega_i x_{2i})\delta_{k,2i}
}{\cosh(\omega_i x_{2i-1}) \cosh(\omega_i x_{2i})}\Biggr]\\
&= \sum_{i=1}^d\Bigl(\tanh(\omega_i x_{2i-1}) \delta_{k,2i-1}
+ \tanh(\omega_i x_{2i}) \delta_{k,2i}\Bigr).\\
\end{aligned}
\]
Moreover, the system $\dot x(t)=\sigmavec(Ax(t))$  written by components becomes, for $k=1 \dots n$,
\[
\dot x_k(t)\!=\!\sum_{i=1}^d \Bigl(\tanh(-\omega_i x_{2i}(t)) \delta_{k,2i-1}
+ \tanh(\omega_i x_{2i-1}(t)) \delta_{k,2i}\Bigr).
\]
Now, if  $t\mapsto x(t)\in\R^{2d}$ is a classical solution of 
$\dot x(t)=\sigmavec(Ax(t))$, then $d H(x(t))/dt$ becomes:
\[
\begin{aligned}
\sum_{k=1}^n
\sum_{i,j=1}^d&\Bigl(\tanh(\omega_i x_{2i-1}(t)) \delta_{k,2i-1}
+ \tanh(\omega_i x_{2i}(t)) \delta_{k,2i}\Bigr)\Bigl(\tanh(-\omega_j x_{2j}(t)) \delta_{k,2j-1}
+ \tanh(\omega_j x_{2j-1}(t)) \delta_{k,2j}\Bigr).
\end{aligned}
\]
In the above expression all the terms of the form 
$\delta_{k,2i-1}\delta_{k,2j}$ or $\delta_{k,2i}\delta_{k,2j-1}$
are $0$, since for no integers $i$ and $j$ it can happen that $2i-1=2j$
or $2j=2i-1$. So what is left is
\begin{equation}
\begin{aligned}
& \sum_{i,j=1}^d \tanh(\omega_i x_{2i-1}(t)) 
\tanh(-\omega_j x_{2j}(t))\delta_{2i-1,2j-1}+ \tanh(\omega_i x_{2i}(t)) 
\tanh(\omega_j x_{2j-1}(t))\delta_{2i,2j}\\
=&\sum_{i=1}^d -\tanh(\omega_i x_{2i-1}(t)) 
\tanh(\omega_i x_{2i}(t))+\tanh(\omega_i x_{2i}(t)) 
\tanh(\omega_i x_{2i-1}(t)) = 0.
\end{aligned}
\end{equation}

\end{proof}
This result, that is a generalization of the  $n=2$ case, shows through Corollary~\ref{cor:sublevels} that also for generic dimension $n$ the hyperbolic 
tangent with skew-symmetric matrix $A$ gives rise to perpetual oscillatory behaviors.

\section{Numerical Simulations}\label{sec:exp}
 In this section we describe the experimental campaign we conducted to provide a more tangible proof of how RNNs with skew-symmetric weight matrices can actually oscillate, examining the role of nonlinear activation functions as well as the generation capabilities at higher hidden dimensions.

 \medskip
\begin{figure*}[t]
    \centering
    \begin{minipage}{0.9\textwidth}
    \begin{minipage}{0.89\textwidth}  % Main figure grid (3 rows, 3 columns)
        \centering
                \includegraphics[width=0.32\linewidth,trim={1.2cm 0.4cm 1.2cm 0.4cm},clip]{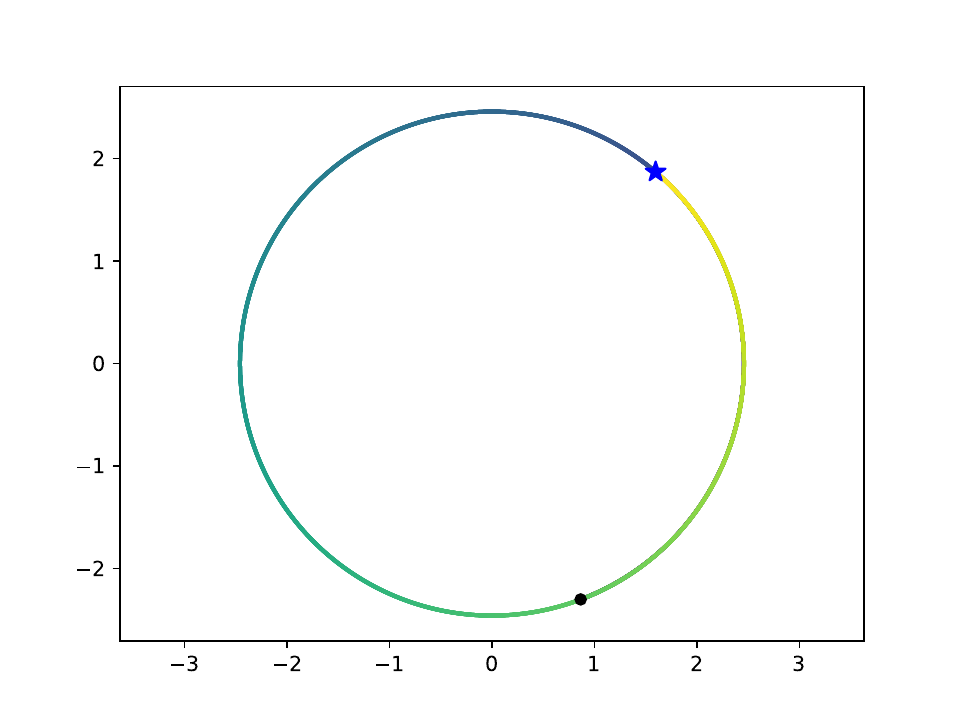}
                \includegraphics[width=0.32\linewidth,trim={1.2cm 0.4cm 1.2cm 0.4cm},clip]{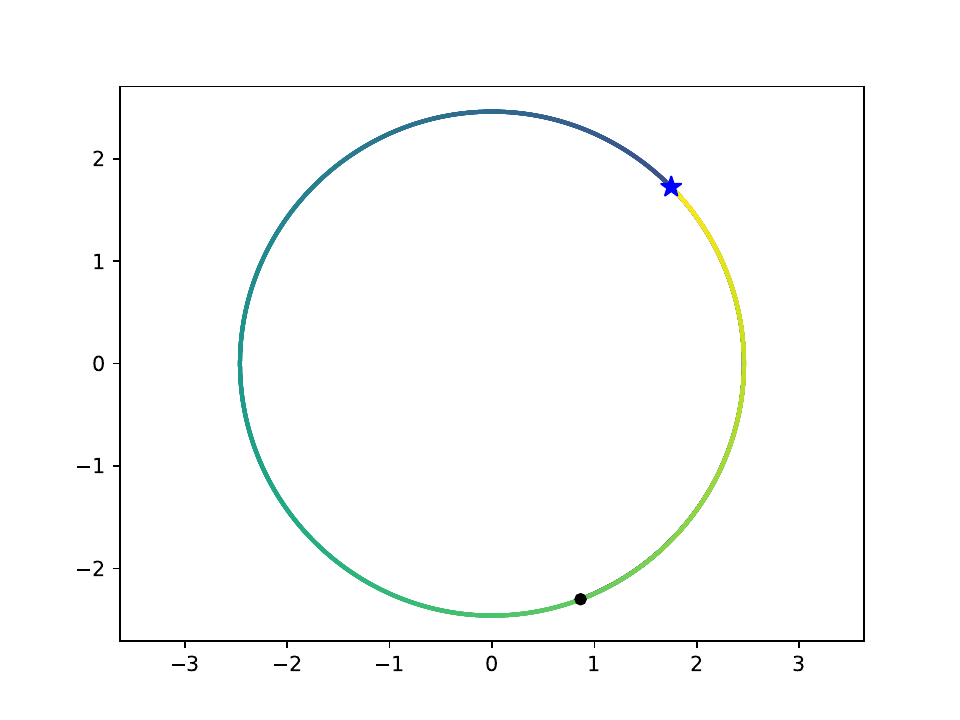}
                \includegraphics[width=0.32\linewidth,trim={1.2cm 0.4cm 1.2cm 0.4cm},clip]{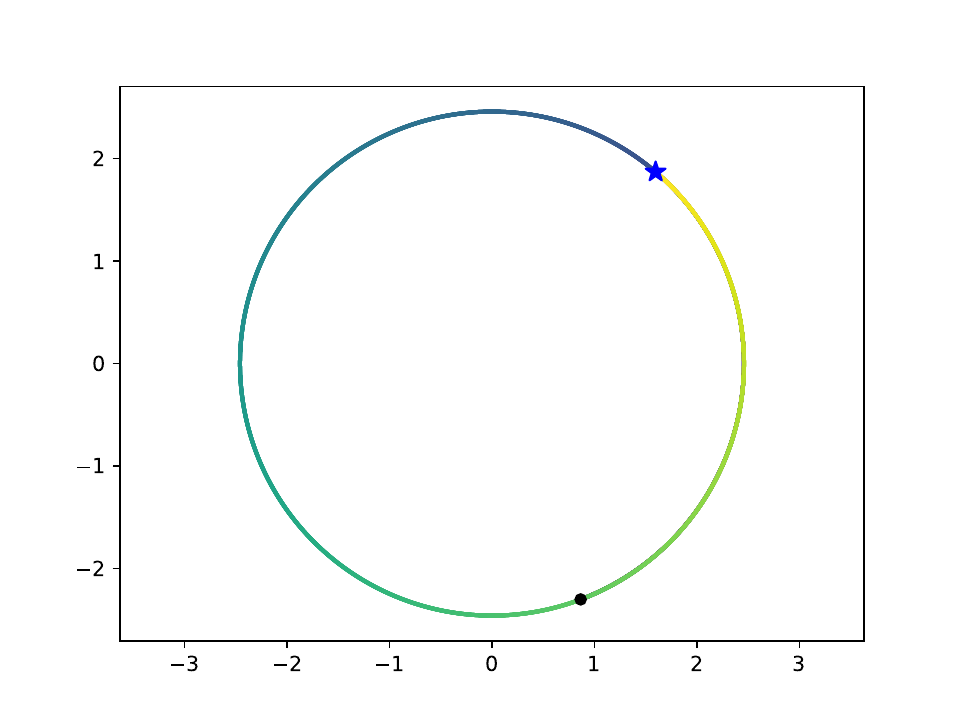}\\
                \includegraphics[width=0.32\linewidth,trim={1.2cm 0.4cm 1.2cm 0.4cm},clip]{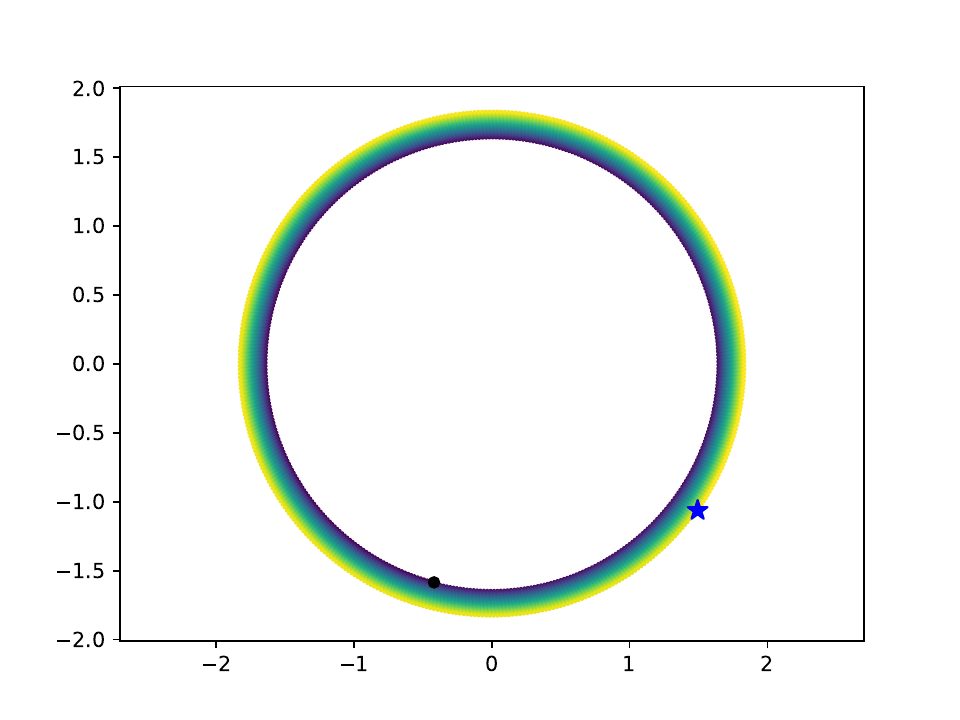}
                \includegraphics[width=0.32\linewidth,trim={1.2cm 0.4cm 1.2cm 0.4cm},clip]{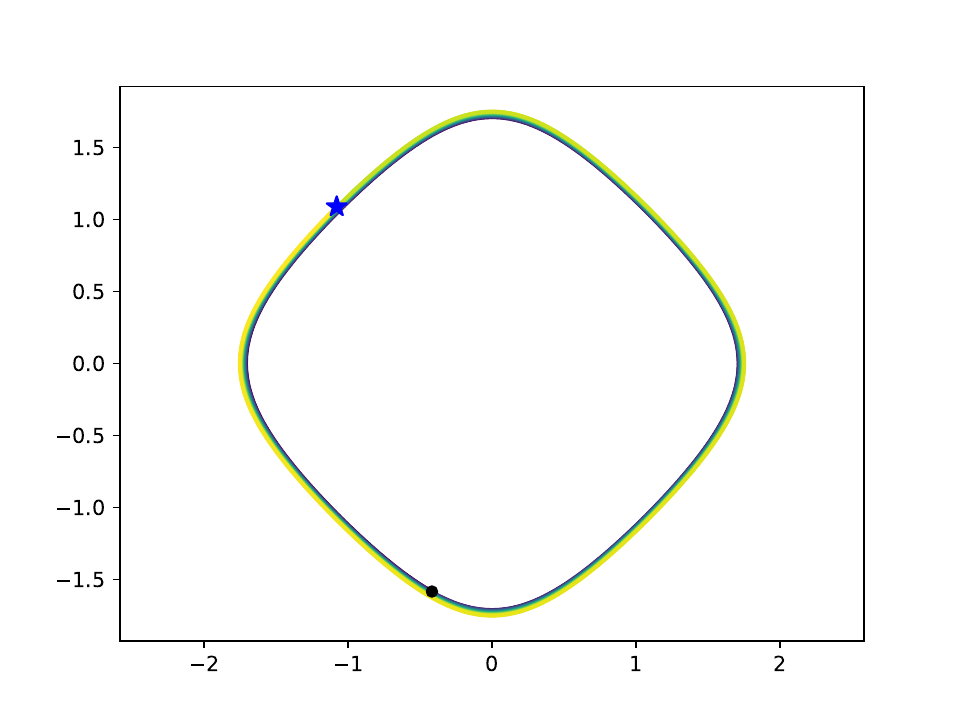}
                \includegraphics[width=0.32\linewidth,trim={1.2cm 0.4cm 1.2cm 0.4cm},clip]{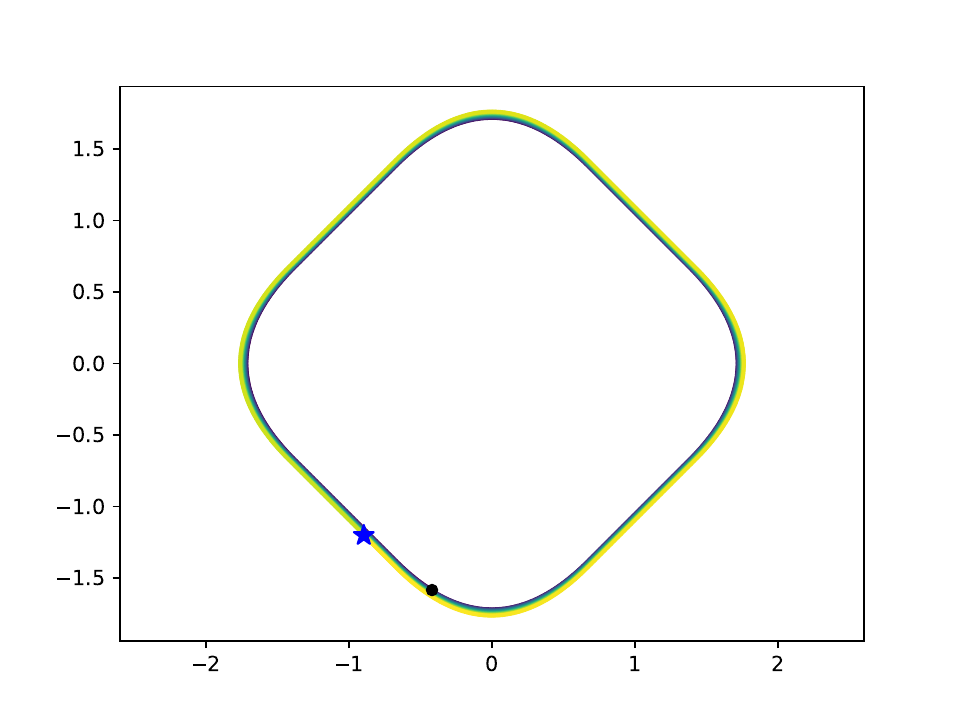}\\
                \includegraphics[width=0.32\linewidth,trim={1.2cm 0.4cm 1.2cm 0.4cm},clip]{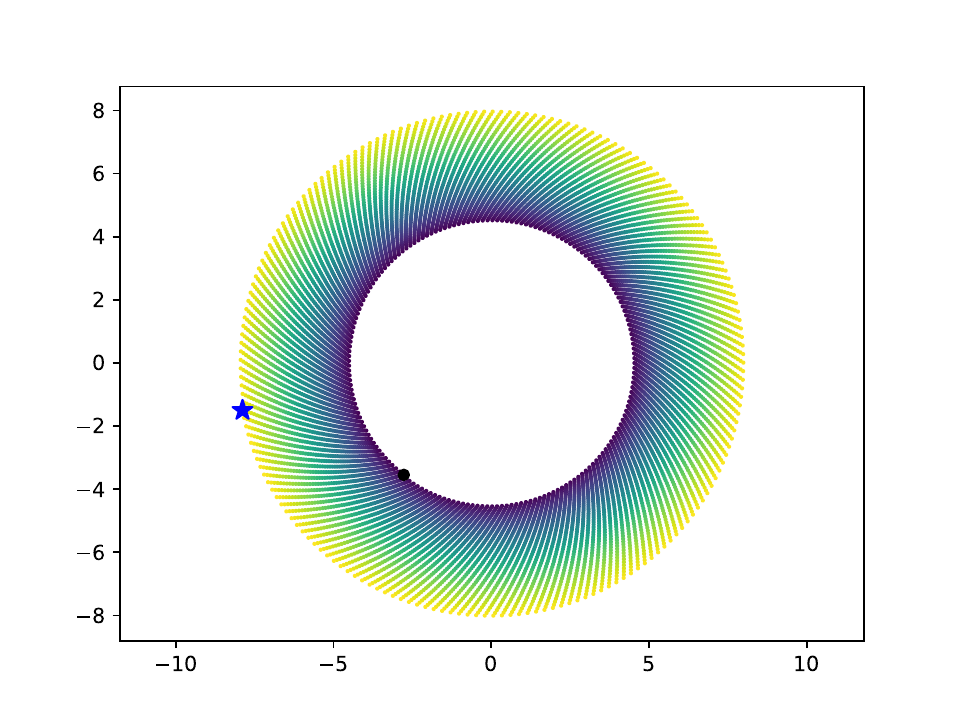}
                \includegraphics[width=0.32\linewidth,trim={1.2cm 0.4cm 1.2cm 0.4cm},clip]{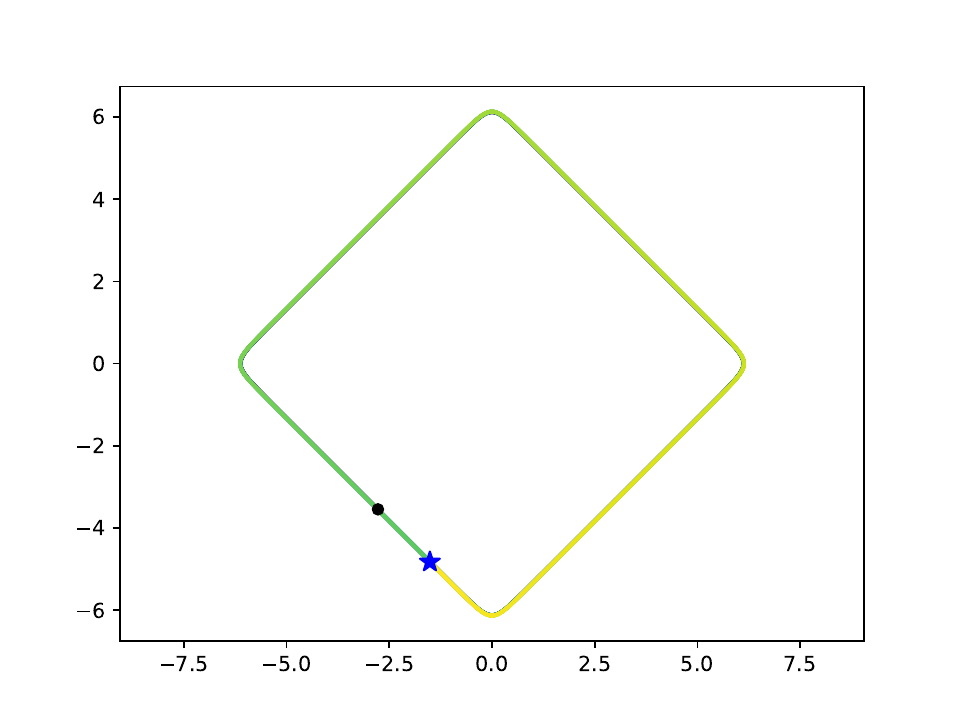}
                \includegraphics[width=0.32\linewidth,trim={1.2cm 0.4cm 1.2cm 0.4cm},clip]{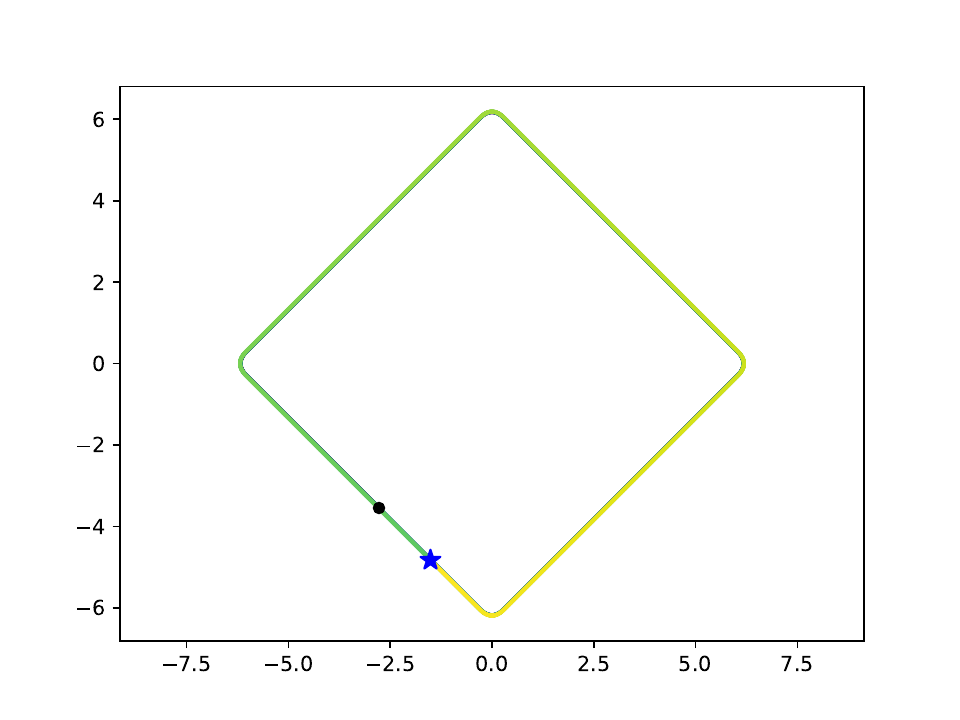}\\
    \end{minipage}
    \hfill
    \begin{minipage}{0.1\textwidth}
        \centering
            \raisebox{1.8cm}{\includegraphics[width=\linewidth,trim={0 16cm 0 0},clip]{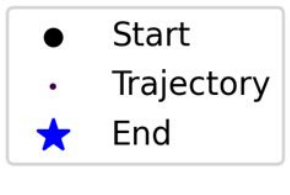}}\\
            \raisebox{3.0cm}{\includegraphics[width=\linewidth,trim={0 0 0 3cm},clip]{figs/barplot_legend.pdf}}
    \end{minipage}
    \end{minipage}
    
    \caption{State-space (2D). The different limit cycles associated with different activation functions. On the \emph{left} column the linear case; in the \emph{center} the tanh; on the \emph{right} the hard-tanh. Each axis is a component of the state, the black dot and the blue star are the initial and final state, respectively. The state trajectory is enriched by a color gradient to suggest the evolution of the system through time, with violet points representing the state during the first steps of the simulation, and the yellow being the last ones, as also suggested by the bar on the right of the grid of images.}
    \label{fig:fig1}
\end{figure*}

\parafango{2-D limit cycles}
 Here  we primarily focus on the two dimensional case of Eq.~\eqref{eq:ctrnn-no-input}, for providing a visual insight of the oscillatory behaviors that emerge from the exploitation of skew-symmetric matrices related to limit cycles in the state space. Additionally, we examine the influence of nonlinear activation functions as illustrated in Figure~\ref{fig:fig1}. More specifically, our experiments were conducted by randomly initializing the skew-symmetric matrix $A$ from a normal distribution $\mathcal{N}\left(0,\tilde{w}\right)$; similarly, the state was sampled from the normal distribution $\mathcal{N}\left(0,\tilde{x}\right)$. We systematically varied both the random seed and the standard deviations, choosing $\tilde{w} \in \left\{0.1, 1, 10\right\}$ and $\tilde{x} \in \left\{1, 10\right\}$, exploiting also high values of the variance to promote the emergence of nonlinear behaviors in both the quasi-linear and saturated regimes of non-linear activation functions. The state was then evolved integrating Eq.~\eqref{eq:ctrnn-no-input} with the forward Euler method in Eq.~\ref{eq:ctrnn-euler} for $100000$ iterations, with a discretization step $\tau=0.001$s. To avoid diverging behaviors, all the simulations where prematurely interrupted if we met the condition $\|x\|_2 \geq 100$, reporting the last value of the state before stopping. The results of our analysis are summarized in Figure~\ref{fig:fig1}, where we compare three different scenarios: the linear case in the left column, the case with the hyperbolic tangent (tanh) activation function in the center column, and its rectified piece-wise linear counterpart (hardtanh), as defined in Eq.~\eqref{eq:hardtanh}, in the right column. In each row the experiments were conducted starting from the same initial state and with same matrix $A$, just varying the activation function. The figures were selected to provide the reader a straightforward visualization of the impact of odd non-linear activation functions, from the case in which they operate in their linear regime (first row) to cases of near-complete saturation (third row). In the first row, the neuron activation $Ax$ remains within the linear region of the activation functions across all cases. This is evident from the identical behavior of the non-linear and linear scenarios. The second row highlights subtle differences between the tanh (center) and the hardtanh (right), where the piecewise linearity of the latter preserves key features of the linear case while still introducing non-linear dynamics. The third row illustrates an extreme saturation regime, where both the tanh and its rectified variant yield the same state trajectory. Notably, the central column empirically confirms previously established theoretical results regarding the non-linear case with a tanh activation function. The reported curves precisely match those described by Eq.~\eqref{eq:energy-2d} and sketched in Figure~\ref{fig:level-sets}. These findings indicate that employing odd non-linear activation functions does not suppress the inherent oscillatory behaviors observed in Linear Recurrent Neural Networks with skew-symmetric weight matrix. Although the related limit cycles are no longer circular, they persist in the state space. Furthermore, non-linear activation functions appear to mitigate the well-documented instability of the forward Euler method \cite{chang2018antisymmetricrnn}. This stabilization effect is particularly evident in the second and third rows of Figure~\ref{fig:fig1}, where the unbounded growth observed in the linear case is significantly reduced. This suggests that such functions enhance the numerical stability of dynamical system integration. In practice, odd non-linear functions preserve limit cycles in the state space observed in the linear case with skew-symmetric matrix $A$, empirically confirming our thorough analysis conducted in Section~\ref{sec:non-linear-ctrnns}. To further investigate this aspect, we tested other well-known non-odd functions, specifically sigmoid and ReLU, under the same experimental setup. The results are shown in Figure~\ref{fig:fig2}, where the linear reference on the left is compared with the other two cases. Unlike odd functions, sigmoid (center) and ReLU (right) disrupt the closed orbits characteristic of the linear system, leading to divergence. The ReLU case even met the stopping condition, visually indicated a red cross to signal its final state. This divergence can be attributed to the fact that both functions are non-negative, explaining the diverging behaviors exhibited from the non-linear dynamical system.

\begin{figure*}[h!]
\centering
\begin{minipage}{0.7\textwidth}
\includegraphics[width=0.32\linewidth,trim={1.2cm 0.4cm 1.2cm 0.4cm},clip]{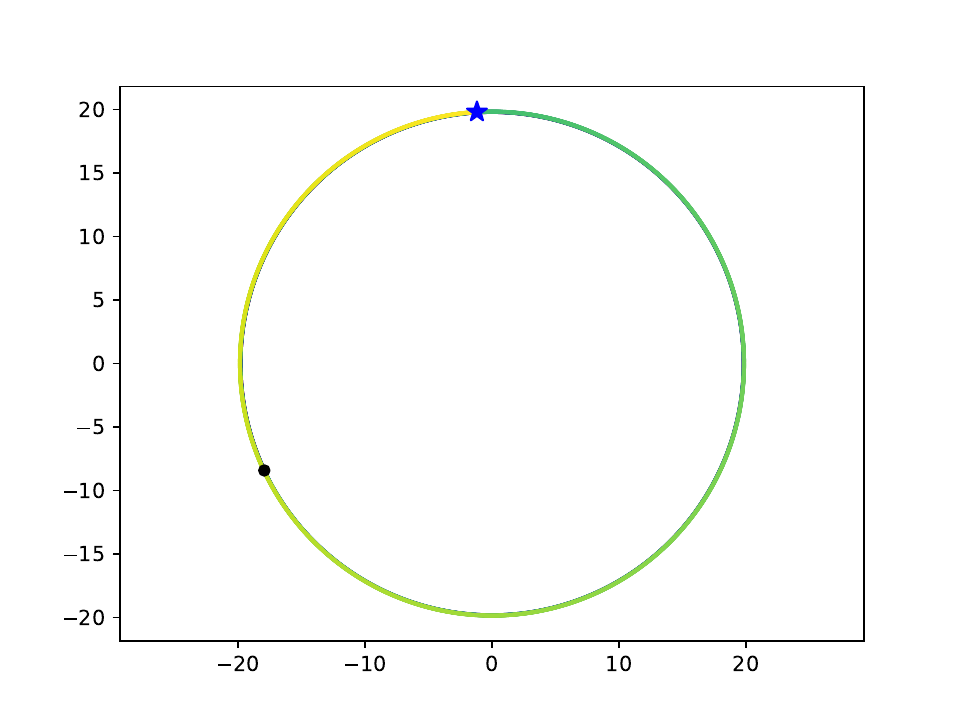}
\includegraphics[width=0.32\linewidth,trim={1.2cm 0.4cm 1.2cm 0.4cm},clip]{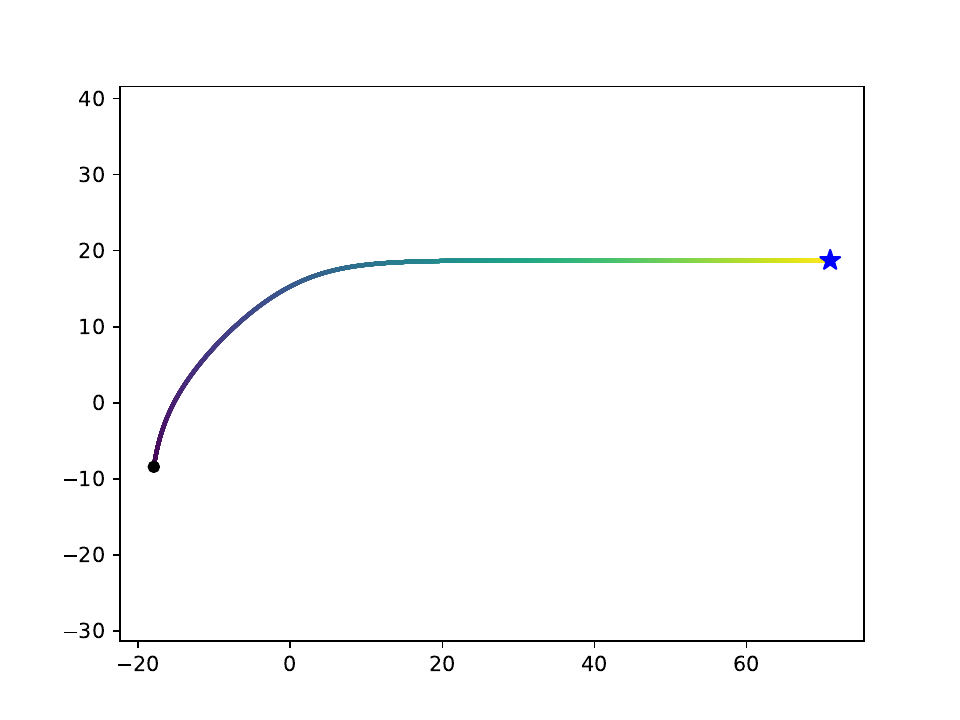}
\includegraphics[width=0.32\linewidth,trim={1.2cm 0.4cm 1.2cm 0.4cm},clip]{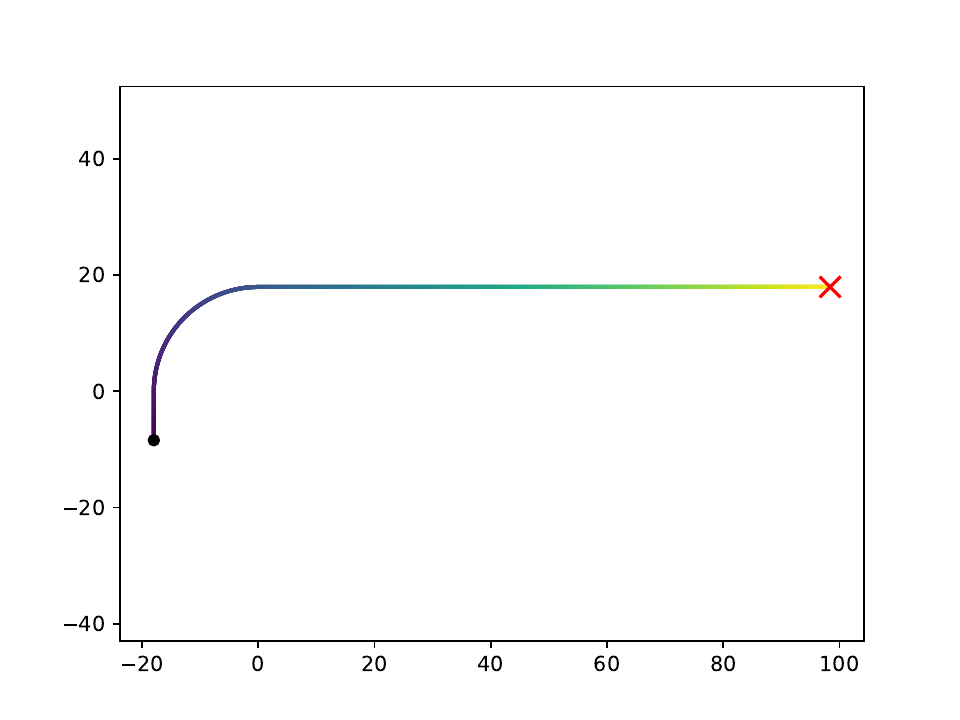}
\end{minipage}
\begin{minipage}{0.1\textwidth} 
\includegraphics[width=0.47\linewidth,trim={0 0 0 3cm},clip]{figs/barplot_legend.pdf}
\end{minipage}
\\
\includegraphics[width=0.27\linewidth,trim={0 4.5cm 0 0},clip]{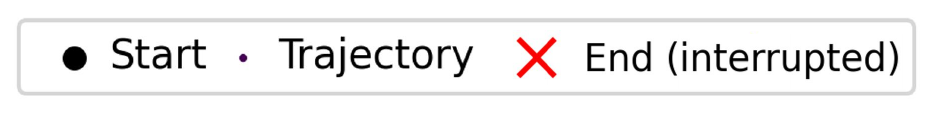}
\caption{State-space (2D). The role of non-odd activation functions like the sigmoid (\emph{center}) and the ReLU (\emph{right}). In the last case the simulation was interrupted before its end because the norm of the state exceeded the threshold, with the red cross marking the value of the state right before the interruption.}
\label{fig:fig2}
\end{figure*}

\parafango{2-D frequency analysis}
The analysis of the frequency spectrum of the 2-D case reported in Figure~\ref{fig:Three-nonlin-spectrum} illustrates the different role of linear, tanh, and hardtanh activation functions. The figure shows a logarithmic diagram of the amplitude response of two dimensional skew-symmetric RNNs with selected frequencies $\omega\in\left\{5\text{Hz},10\text{Hz},15\text{Hz},20\text{Hz},25\text{Hz}\right\}$, characterized by different colors in each graph. In the linear case on the left, selecting the frequency directly translates to an almost perfect line spectrum, as expected. The second diagram from the left tells us a completely different story when employing the tanh activation function, leading to similar responses for all the selected frequencies. This suggests that, in the nonlinear case, the relation between the weights of the RNN and its frequency profile is non-trivial, thus making it difficult to control its modes via their parameters. The last two images show that the piece-wise linear hardtanh response sits in the middle between the last two cases. Interestingly, a multiplicative increment factor $w_m$ of the weights helps the emergence of linear-like behaviors in the obtained spectrum. More specifically, we can observe that increasing the value from $w_m=10$ (third plot) to $w_m=25$ (fourth) enlarges the spectral band in which the system behaves linearly. Finally, the case $w_m=50$ is not shown in the figure since it completely yield the same linear case of the leftmost plot.

\begin{figure*}[htbp]
    \centering
        \includegraphics[width=0.24\linewidth,trim={0.8cm 0.6cm 0.cm 0.cm},clip]{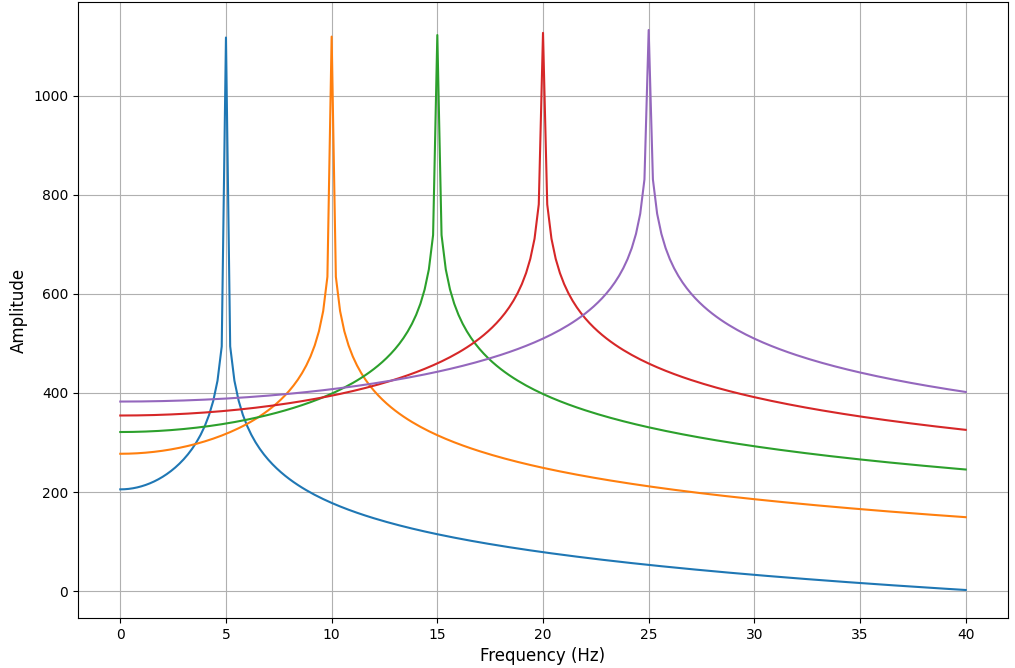}
        \includegraphics[width=0.24\linewidth,trim={0.8cm 0.6cm 0.cm 0.cm},clip]{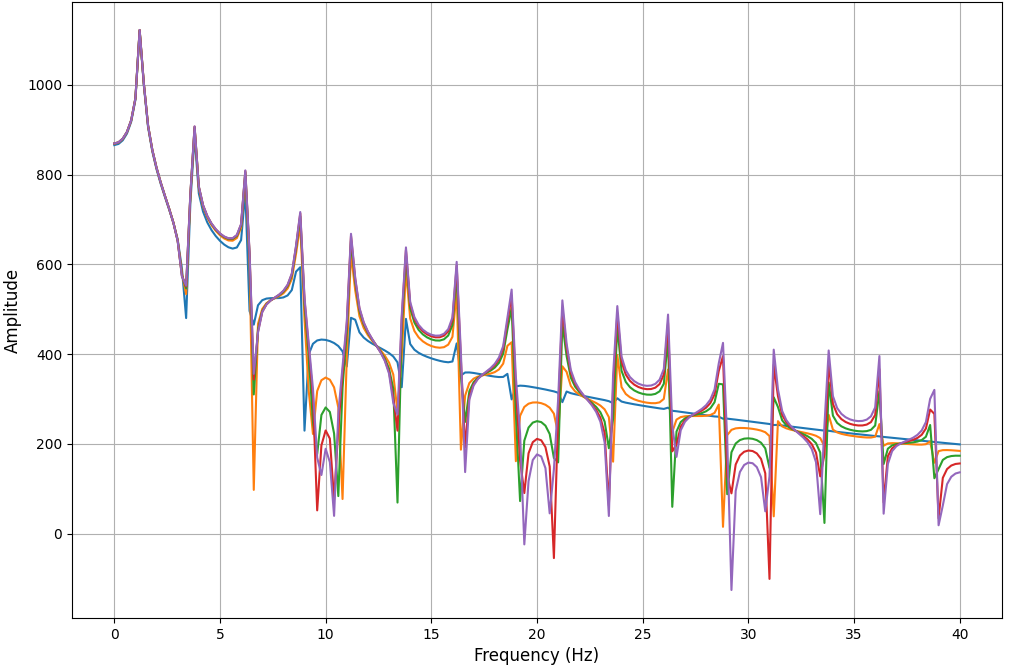}
        \includegraphics[width=0.24\linewidth,trim={0.8cm 0.6cm 0.cm 0.cm},clip]{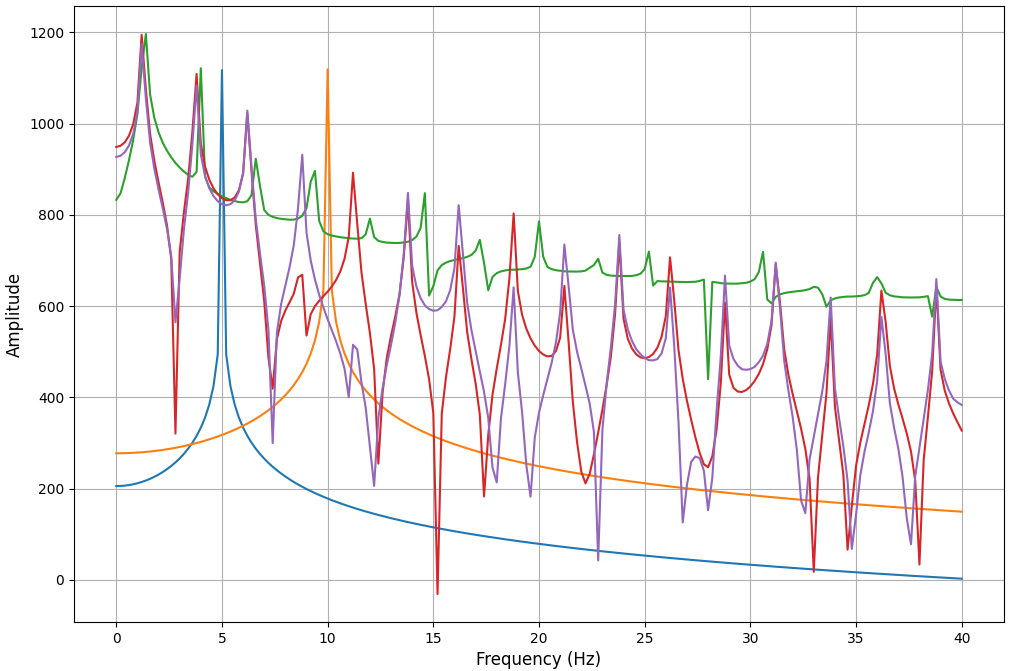}
        \includegraphics[width=0.24\linewidth,trim={0.8cm 0.6cm 0.cm 0.cm},clip]{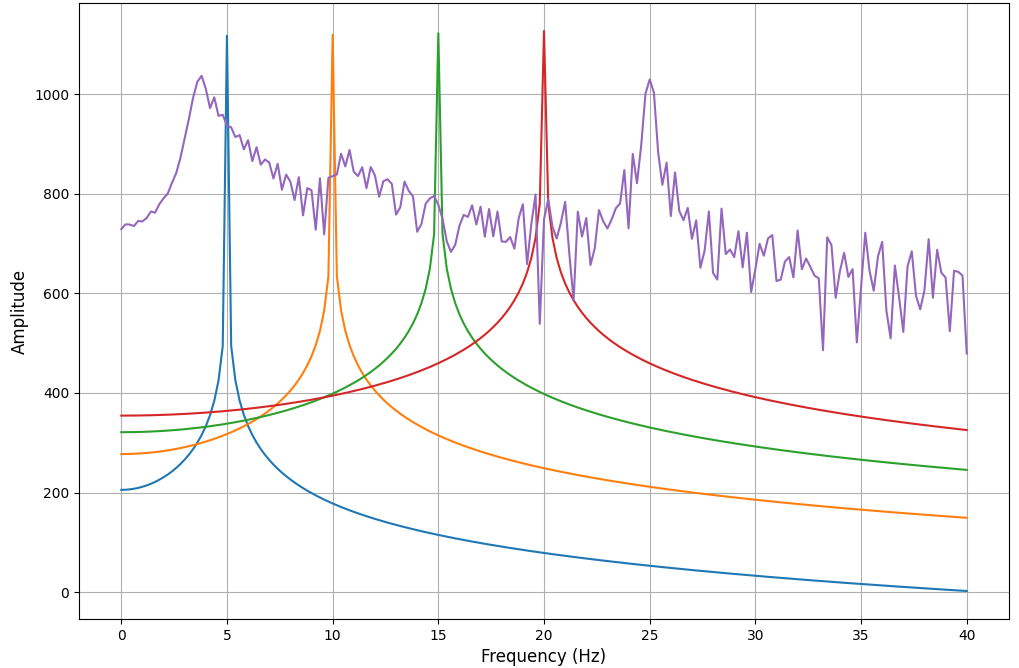}
    \caption{Spectral response of skew-symmetric 2x2 RNN with selected angular frequencies (different colors). The y-axis is the amplitude (log-scale), while the x-axis is the frequency (Hz). The first two plots from the left refer to the linear and the tanh cases, respectively. Still proceeding \emph{left-to-right} we can see the case of hard-tanh activation function for $w_m=10$ and $w_m=25$ respectively.}
    \label{fig:Three-nonlin-spectrum}
\end{figure*}

\parafango{High-dimensional generation}
We now investigate the process of generation of recurrent neural networks in the interesting case of higher hidden dimensions. The theoretical results given in the paper are confirmed in our simulation: skew-symmetric matrices always yield generation, even though Euler's discretization issues arises (see Fig.~\ref{fig:role_of_neurons}-left). Increasing the number of hidden neurons $n$ provides the network a more complex spectrum, which yields the generation of less trivial signals, as shown by the last two images. Here the value of $n$ is 20 and 40 respectively, and we can clearly see that both produce a response which is no longer a simple sinusoid. We recall that in the two dimensional case, a circular path with constant angular velocity on the state space results in a sinusoidal evolution of both the state components; increasing the number of hidden units provides the network the ability to produce oscillations with more frequency components, as also clearly depicted by the STFT spectrums in the second row of the image. Interestingly, these experiments show that increasing the number of neurons leads to better numerical stability, while the reduction in the fundamental frequency of the generated spectrum is related to the fact that the weight matrix $W$ was initialized in a uniform range inversely proportional to the number of connections.

\section{Conclusions}\label{sec:concl}
This study explores the continuous-time dynamics of RNNs, with an emphasis on systems characterized by nonlinear activation functions. Our findings demonstrate that skew-symmetric weight matrices play a crucial role in enabling perpetual oscillatory behavior, in both linear and nonlinear RNNs. Moreover, the usage of odd, bounded and continuous activation functions, like the hyperbolic tangent and its rectified counterpart, preserves these oscillatory patterns by maintaining limit cycles in state space. We observed that nonlinear activation functions not only preserve the oscillatory nature of these networks, but also contribute to enhance the numerical stability of dynamical system integration through forward Euler method. On the other hand, non-odd activation functions such as sigmoid and ReLU disrupt the closed orbits inherent to the linear case, often resulting in divergent behavior. These results reinforce the importance of activation function selection in maintaining dynamic stability and achieving desired oscillatory properties in RNN models. In conclusion, this work highlights the interplay between system architecture and dynamic behavior in RNNs, paving the way for more robust implementations in applications where stable, oscillatory behavior is crucial.

\section{Acknoweldgements}
This work was supported by the University of Siena (Piano per lo Sviluppo della Ricerca - PSR 2024, F-NEW FRONTIERS 2024), under the project ``TIme-driveN StatEful Lifelong Learning'' (TINSELL) and also by the project ``CONSTR: a COllectionless-based Neuro-Symbolic Theory for learning and Reasoning'', PARTENARIATO ESTESO ``Future Artificial Intelligence Research - FAIR'', SPOKE 1 ``Human-Centered AI'' Universit\`a di Pisa,  ``NextGenerationEU'', CUP I53C22001380006.

\begin{figure*}[h]
    \raggedright
        \includegraphics[width=0.3\linewidth,trim={0.8cm 0.6cm 0.2cm 0.1cm},clip]{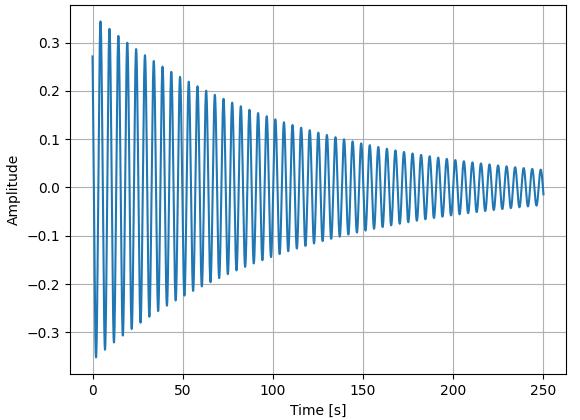}
        \includegraphics[width=0.3\linewidth,trim={0.8cm 0.6cm 0.cm 0.cm},clip]{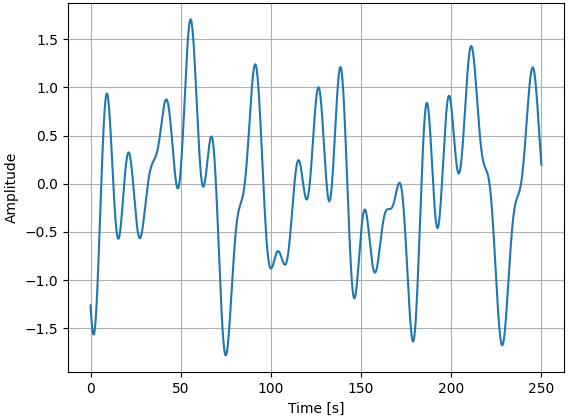}
        \includegraphics[width=0.3\linewidth,trim={0.6cm 0.6cm 0.cm 0.cm},clip]{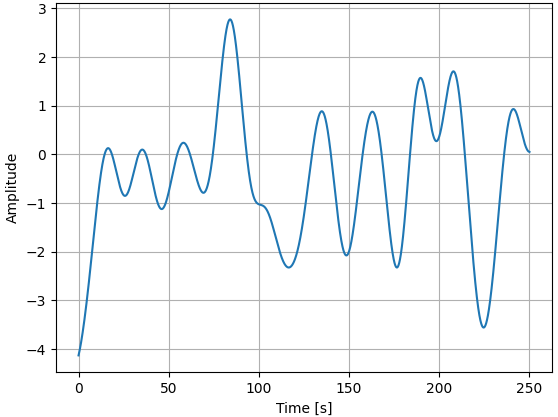}
    \vskip\baselineskip
        \hskip 2mm \includegraphics[width=0.287\linewidth,trim={0.8cm 0.8cm 5.2cm 0.9cm},clip]{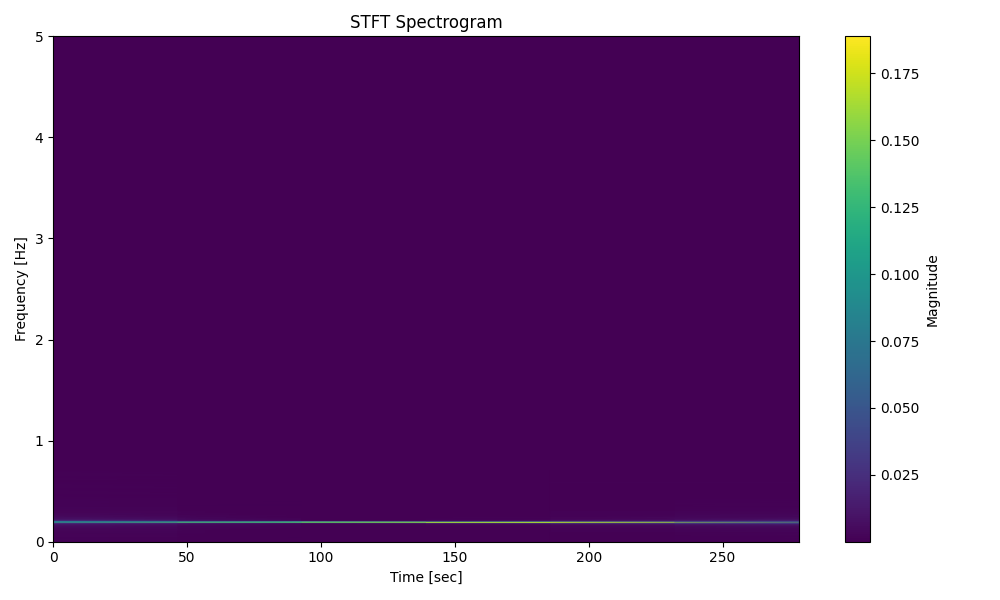}
        \hskip 2mm
        \includegraphics[width=0.287\linewidth,trim={0.6cm 0.4cm 3.6cm 0.6cm}, clip]{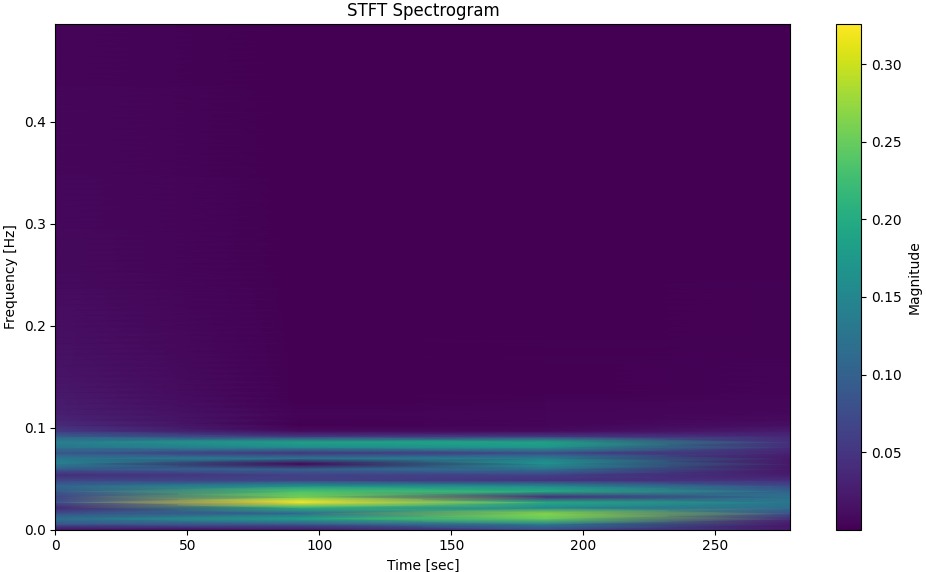}
    \hskip 2mm
        \includegraphics[width=0.332\linewidth,trim={0.6cm 0.6cm 0.cm 0.5cm},clip]{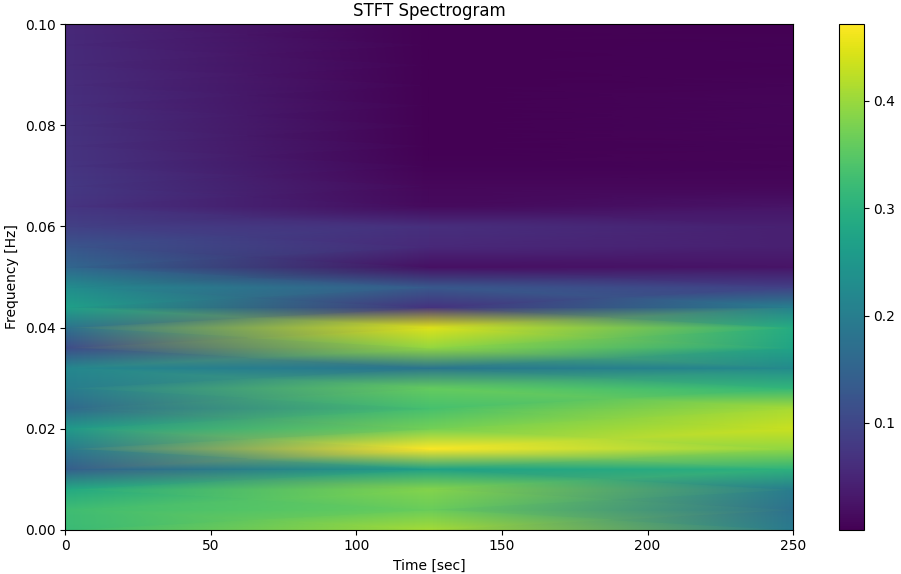}
    \caption{Instability arising from the Euler's discretization (\emph{top left}) is mitigated by increasing RNN with skew-symmetric matrices and tanh activation function ensures perpetual generation; however, issues arise with Euler's discretization (\emph{left}). The neural response (top-row, x-axis: time, y-axis: output) is shown to increase as the number of neurons grows (the hidden dimensions are $n=\left\{2,20,40\right\}$ from left to right). As the dimension increases, numerical stability improves, and the spectrum becomes richer but remains concentrated at lower frequencies. The lower row is the spectrum of the generated signal computed with the Short-time FFT, with yellower points highlighting frequencies with bigger coefficients (x-axis: time, y-axis: frequency (Hz)).}
\label{fig:role_of_neurons}
\end{figure*}

%\printbibliography
\bibliographystyle{plain}
\bibliography{main}

\end{document}